\documentclass[letterpaper, 10 pt, conference]{class/ieeeconf}  

\IEEEoverridecommandlockouts                              
\usepackage{xcolor}
\usepackage{siunitx}  
\usepackage{graphicx}
 \graphicspath{./figures/robot_demonstration}

\usepackage{amsmath}
\usepackage{amssymb}
\usepackage{latexsym}
\usepackage{url}
\usepackage{cite}
\usepackage{relsize}
\usepackage{multirow}
\usepackage{afterpage}
\usepackage{ifthen}
\usepackage{graphicx}
\usepackage{algpseudocode,algorithm,algorithmicx}
\usepackage{subfigure}
\usepackage{epstopdf}
\usepackage{stackengine}
\usepackage{textcomp} 

\usepackage{gensymb}
\usepackage{booktabs}
\usepackage{makecell}
\usepackage{hhline}
\usepackage{courier}
\usepackage{lipsum}
\usepackage{bm}
\usepackage{flushend}

\providecommand{\Comma}{\text{~,\xspace}}

\usepackage{xcolor} 
            
\makeatletter
\let\NAT@parse\undefined
\makeatother
\usepackage[bookmarks=false, linkcolor=blue, urlcolor=blue, citecolor=blue]{hyperref} 

\hypersetup{
    colorlinks=true,
    linkcolor=red,
    filecolor=magenta,      
    urlcolor=blue,
    pdfstartview={FitH},
    citecolor =blue
    }
    
\usepackage{tikz}
\usetikzlibrary{arrows.meta,positioning,fit,backgrounds,calc}

\DeclareMathOperator*{\argmin}{arg\,min}
\usepackage{xspace}

\providecommand{\kuka}{\textsc{KUKA} LBR iiwa R820\xspace}

\graphicspath{{./figures/}}

\title{\LARGE \bf 
Coupled Routing and Configuration Optimization for Multi-Viewpoint Robotic Inspection}

\author{Minh Nhat Vu$^{1,2}$, Khang Nguyen$^{2,3}$, Vu Trung Tran$^{2}$, Vien Ngo$^{2}$
\thanks{$^{1}$ Center for Artificial Intelligence Research (CAIR), VinUniversity, Hanoi, Vietnam {\tt\small minh.vn@vinuni.edu.vn},}%
\thanks{$^{2}$ VinRobotics JSC, Hanoi, Vietnam,
}%
\thanks{$^{3}$ Division of Computer and Mathematical Sciences, Mohamed bin Zayed University of Artificial Intelligence (MBZUAI), Abu Dhabi, UAE. 
}
}

\begin{document}

\newtheorem{problem}{Problem}
\newtheorem{lemma}{Lemma}
\newtheorem{theorem}[lemma]{Theorem}
\newtheorem{claim}{Claim}
\newtheorem{corollary}[lemma]{Corollary}
\newtheorem{definition}[lemma]{Definition}
\newtheorem{proposition}[lemma]{Proposition}
\newtheorem{remark}[lemma]{Remark}
\newenvironment{LabeledProof}[1]{\noindent{\it Proof of #1: }}{\qed}

\def\beq#1\eeq{\begin{equation}#1\end{equation}}
\def\bea#1\eea{\begin{align}#1\end{align}}
\def\beg#1\eeg{\begin{gather}#1\end{gather}}
\def\beqs#1\eeqs{\begin{equation*}#1\end{equation*}}
\def\beas#1\eeas{\begin{align*}#1\end{align*}}
\def\begs#1\eegs{\begin{gather*}#1\end{gather*}}

\newcommand{\poly}{\mathrm{poly}}
\newcommand{\eps}{\epsilon}
\newcommand{\e}{\epsilon}
\newcommand{\polylog}{\mathrm{polylog}}
\newcommand{\rob}[1]{\left( #1 \right)} 
\newcommand{\sqb}[1]{\left[ #1 \right]} 
\newcommand{\cub}[1]{\left\{ #1 \right\} } 
\newcommand{\rb}[1]{\left( #1 \right)} 
\newcommand{\abs}[1]{\left| #1 \right|} 
\newcommand{\zo}{\{0, 1\}}
\newcommand{\zonzo}{\zo^n \to \zo}
\newcommand{\zokzo}{\zo^k \to \zo}
\newcommand{\zot}{\{0,1,2\}}
\newcommand{\en}[1]{\marginpar{\textbf{#1}}}
\newcommand{\efn}[1]{\footnote{\textbf{#1}}}
\newcommand{\vecbm}[1]{\boldmath{#1}} 
\newcommand{\uvec}[1]{\hat{\vec{#1}}}
\newcommand{\thv}{\vecbm{\theta}}
\newcommand{\junk}[1]{}
\newcommand{\var}{\mathop{\mathrm{var}}}
\newcommand{\rank}{\mathop{\mathrm{rank}}}
\newcommand{\diag}{\mathop{\mathrm{diag}}}
\newcommand{\tr}{\mathop{\mathrm{tr}}}
\newcommand{\acos}{\mathop{\mathrm{acos}}}
\newcommand{\atantwo}{\mathop{\mathrm{atan2}}}
\newcommand{\SVD}{\mathop{\mathrm{SVD}}}
\newcommand{\quadf}{\mathop{\mathrm{q}}}
\newcommand{\linterp}{\mathop{\mathrm{l}}}
\newcommand{\sgn}{\mathop{\mathrm{sign}}}
\newcommand{\sym}{\mathop{\mathrm{sym}}}
\newcommand{\avg}{\mathop{\mathrm{avg}}}
\newcommand{\mean}{\mathop{\mathrm{mean}}}
\newcommand{\erf}{\mathop{\mathrm{erf}}}
\newcommand{\grad}{\nabla}
\newcommand{\R}{\mathbb{R}}
\newcommand{\defeq}{\triangleq}
\newcommand{\dims}[2]{[#1\!\times\!#2]}
\newcommand{\sdims}[2]{\mathsmaller{#1\!\times\!#2}}
\newcommand{\udims}[3]{#1}
\newcommand{\udimst}[4]{#1}
\newcommand{\com}[1]{\rhd\text{\emph{#1}}}
\newcommand{\ind}{\hspace{1em}}
\newcommand{\floor}[1]{\left\lfloor{#1}\right\rfloor}
\newcommand{\step}[1]{\vspace{0.5em}\noindent{#1}}
\newcommand{\quat}[1]{\ensuremath{\mathring{\mathbf{#1}}}}
\newcommand{\norm}[1]{\left\lVert#1\right\rVert}
\newcommand{\ignore}[1]{}
\newcommand{\specialcell}[2][c]{\begin{tabular}[#1]{@{}c@{}}#2\end{tabular}}
\newcommand*\Let[2]{\State #1 $\gets$ #2}
\newcommand{\algorithmicbreak}{\textbf{break}}
\newcommand{\Break}{\State \algorithmicbreak}
\newcommand{\ra}[1]{\renewcommand{\arraystretch}{#1}}

\renewcommand{\vec}[1]{\mathbf{#1}} 

\algdef{S}[FOR]{ForEach}[1]{\algorithmicforeach\ #1\ \algorithmicdo}
\algnewcommand\algorithmicforeach{\textbf{for each}}
\algrenewcommand\algorithmicrequire{\textbf{Require:}}
\algrenewcommand\algorithmicensure{\textbf{Ensure:}}
\algnewcommand\algorithmicinput{\textbf{Input:}}
\algnewcommand\INPUT{\item[\algorithmicinput]}
\algnewcommand\algorithmicoutput{\textbf{Output:}}
\algnewcommand\OUTPUT{\item[\algorithmicoutput]}

\maketitle
\thispagestyle{empty}
\pagestyle{empty}

\begin{abstract}
We present a unified framework that turns a set of 6-DoF inspection viewpoints into a time-optimal, collision-free route for a 9-DoF robotic system. Unlike modular pipelines that fix a single inverse-kinematics (IK) configuration per viewpoint, build an all-pairs travel-time map, and then route, our method jointly optimizes the visiting order and the per-viewpoint configuration in a single global search. The three-dimensional self-motion manifold of each viewpoint is parameterized in closed form so that the pose constraint holds by construction, the rest-to-rest travel time is approximated by a closed-form admissible double-integrator surrogate, and the tour is encoded by random keys. A derivative-free optimizer (CMA-ES) minimizes a cheap penalized objective over order and configuration, after which direct-collocation trajectory optimization is applied only to the edges of the selected route to certify dynamic feasibility and torque limits, and to return exact timings. This reduces the trajectory solves from quadratic to linear in the number of viewpoints and removes the decoupling that prevents modular pipelines from being globally time-optimal. Simulations and real-robot experiments on a KUKA LBR iiwa with a 2-DoF linear stage validate feasibility, smooth execution, and reduced end-to-end inspection time relative to modular and naive distance-based baselines.
\end{abstract}


\section{INTRODUCTION} \label{Sec:Intro}
Effectively performing a sequence of manipulation tasks on robots with high degrees of freedom (DoF), e.g., mobile manipulators for surface inspection, is challenging due to the integration of multiple systems, such as task and trajectory planning, and control. 
Many industrial tasks can be reduced to visiting a large set of 6-DoF poses in some order. Unlike assembly or stacking, which follow a fixed sequence, inspection, sorting, and surface finishing leave the order free, and in many cases, the chosen sequence strongly affects throughput.

In a typical object-scanning task, on the order of $100$ viewpoints in 6 DoFs (position and orientation) must be visited by a manipulator, and the challenge is to route through all of them in minimum time under the nonlinear mapping between task and joint space. Such inspection enables paradigms such as zero-defect and flexible manufacturing, in which each product, possibly customized, may require an individual plan generated automatically from CAD data.

In this paper, we depart from the common two-stage recipe of first fixing a configuration per viewpoint and then routing. The nonlinear task-to-joint mapping means that two viewpoints close in task space need not be close in configuration space, so the configuration chosen at a viewpoint should depend on its neighbors in the route. We therefore parameterize the self-motion manifold of each viewpoint in closed form, approximate the rest-to-rest travel time by a closed-form surrogate of the linearized dynamics, and optimize the visiting order together with the per-viewpoint redundancy in a single global search using CMA-ES. The expensive direct-collocation trajectory optimization, which considers obstacles such as the scanned object, virtual walls, and the table, is then run only along the edges of the selected route to certify dynamic feasibility and torque limits, and to return exact timings.

Our contributions are threefold. (i) We formulate multi-viewpoint inspection as a single joint optimization over the visiting order and the per-viewpoint configuration, removing the decoupling of IK, cost-map construction, and routing that prevents modular pipelines from being globally time-optimal. (ii) We make this tractable through a closed-form, redundancy-parameterized IK that satisfies each pose by construction, a closed-form admissible travel-time surrogate from the linearized dynamics, and a random-key encoding that lets one derivative-free optimizer (CMA-ES) search order and configuration together. (iii) We apply direct-collocation trajectory optimization only to the $M-1$ edges of the selected route, reducing trajectory solves from quadratic to linear in the number of viewpoints while certifying dynamic feasibility and torque limits.

\section{Related Work} \label{Sec:rw}
Sampling-based planners such as RRT~\cite{lavalle2001rapidly} offer probabilistic completeness but require solving a two-point boundary value problem that is tractable only for a limited set of systems. Optimization-based planning instead encodes dynamics, state and input bounds, and collision avoidance directly, producing smooth trajectories without post-processing at the cost of a nonlinear, nonconvex problem; point-to-point schemes can also shape secondary objectives such as manipulability~\cite{beck2022singlularity}.

Beyond a single point-to-point motion, inspection requires ordering many viewpoints. For mobile manipulators, most methods place the base to maximize manipulability and then sequence the motions~\cite{xu2020planning}, or cast the order as a traveling salesman problem solved by simulated annealing~\cite{harada2015base}; these typically separate base and arm planning to reduce complexity. Liu et al.~\cite{liu2020optimal} compute optimal paths through given measurement points in two steps: a time matrix followed by routing, but only in Cartesian 3D space. In contrast, we plan over the full 9-DoF system with its dynamics and, crucially, do not fix the per-viewpoint configuration before routing.


\section{Modeling}
\label{sec: modeling}

\subsection{Hardware}
This work considers a mobile-like manipulator task of scanning objects in a cluttered environment. 
The system has 9 DoF, consisting of the 7-DoF \kuka attached to a 2-DoF linear axis system that can move in the $x$- and $ y$-directions. 
The task involves planning an optimal route of smooth, collision-free, and dynamically feasible trajectories passing all given 6-DoF viewpoints (surrounding the objects) in optimal order. 
Since the working environment is restricted, we consider two virtual walls (see Fig. \ref{fig: obstacle avoidance}) as obstacles. 

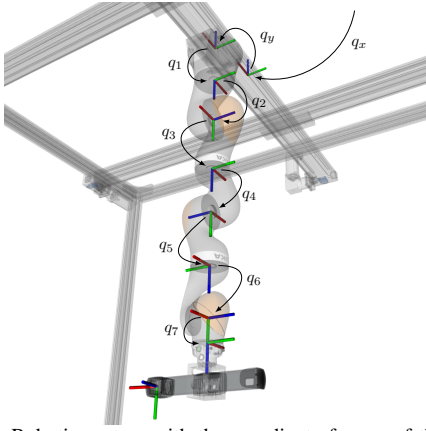
\begin{figure}[h]
\centering
\scalebox{0.62}{
\begin{tikzpicture}[>=latex]
\tikzstyle{frame}=[circle, radius=1]
\coordinate (W)  at ( 3.00, 4.50);
\coordinate (L1) at ( 0.67, 2.94);
\coordinate (L2) at ( 0.00, 3.51);
\coordinate (L3) at ( 0.00, 2.83);
\coordinate (L4) at (-0.02, 2.01);
\coordinate (L5) at (-0.05, 0.93);
\coordinate (L6) at (-0.08, 0.03);
\coordinate (L7) at (-0.11,-1.11);
\coordinate (L8) at (-0.14,-2.22);
\coordinate (L9) at (-0.17,-2.77);
\coordinate (C)  at (-1.25,-3.80);

\node[frame] (OW) at (W) {};
\node[frame] (OL1) at (L1) {};
\node[frame] (OL2) at (L2) {};
\node[frame] (OL3) at (L3) {};
\node[frame] (OL4) at (L4) {};
\node[frame] (OL5) at (L5) {};
\node[frame] (OL6) at (L6) {};
\node[frame] (OL7) at (L7) {};
\node[frame] (OL8) at (L8) {};
\node[frame] (OL9) at (L9) {};
\node[frame] (OC) at (C) {};

\node at (0,0) {\includegraphics[height=9cm]{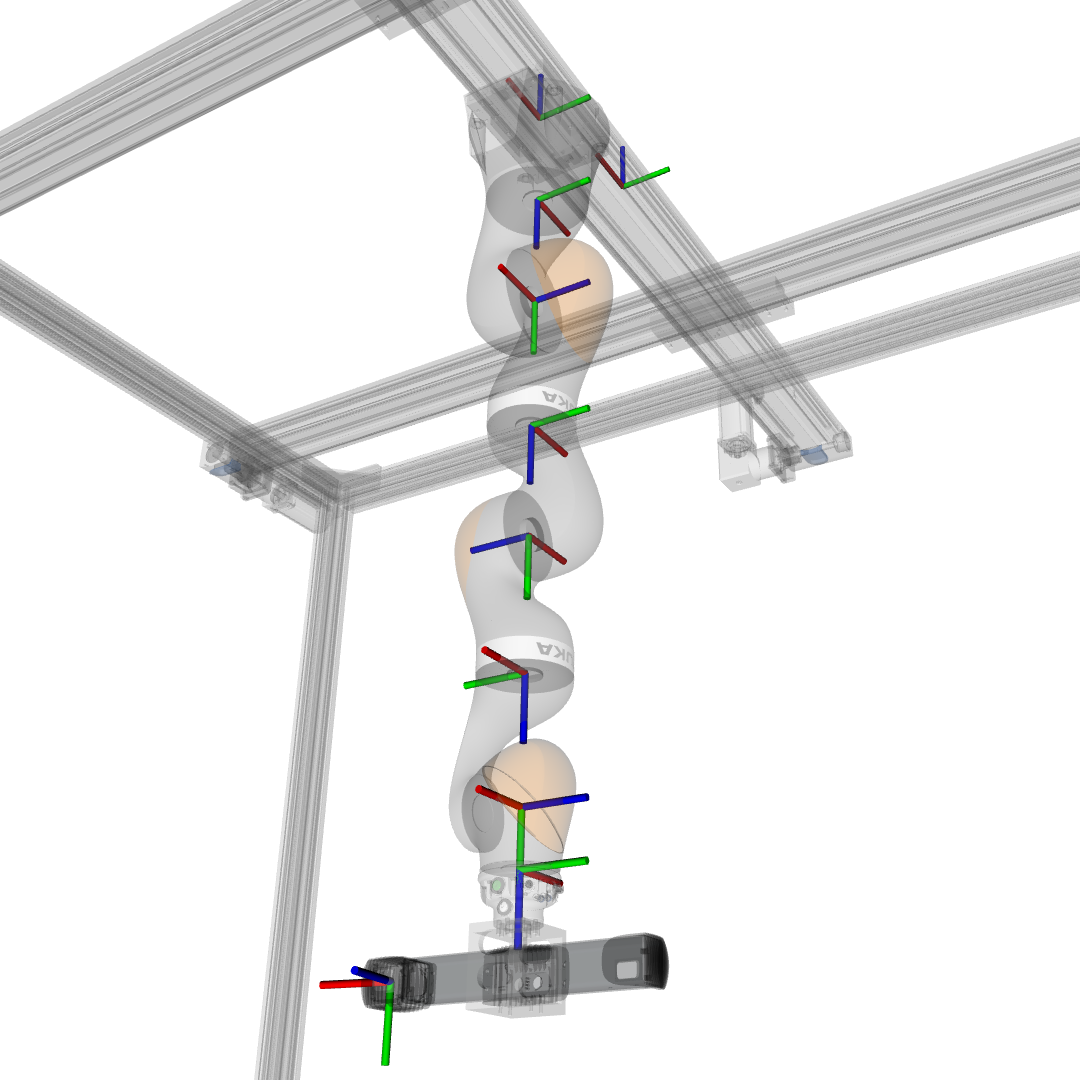}};

\path (OW) edge[->, out=260, in=350, looseness=1] node[pos=0.2, anchor=150]{$q_x$} (OL1);
\path (OL1) edge[->, out=60, in=60, looseness=2] node[pos=0.3, anchor=210]{$q_y$} (OL2);
\path (OL2) edge[->, out=180, in=180, looseness=2] node[pos=0.5, anchor=0]{$q_1$} (OL3);
\path (OL3) edge[->, out=0, in=350, looseness=2] node[pos=0.5, anchor=180]{$q_2$} (OL4);
\path (OL4) edge[->, out=190, in=150, looseness=2] node[pos=0.4, anchor=0]{$q_3$} (OL5);
\path (OL5) edge[->, out=0, in=30, looseness=2] node[pos=0.7, anchor=170]{$q_4$} (OL6);
\path (OL6) edge[->, out=220, in=170, looseness=2] node[pos=0.5, anchor=0]{$q_5$} (OL7);
\path (OL7) edge[->, out=0, in=40, looseness=2] node[pos=0.5, anchor=180]{$q_6$} (OL8);
\path (OL8) edge[->, out=180, in=180, looseness=2] node[pos=0.5, anchor=0]{$q_7$} (OL9);
\end{tikzpicture}
}
\caption{Robotic system with the coordinate frames of the links.}
\label{fig: modeling}
\end{figure}

\subsection{Modeling} The kinematics of the \kuka resemble the structure of a human arm with seven rotational degrees of freedom $\mathbf{q}_r = [q_1,q_2,...,q_7]^\mathrm{T}$, illustrated in Figure \ref{fig: modeling}. Moreover, the robot base $\mathcal{B}$ is mounted onto a linear axis setup with the 2 DoFs $\mathbf{q}_l = [q_x,q_y]^\mathrm{T}$, which can be moved in $x$- and $y$-direction. 

The robot's rigid-body dynamics, derived by the Lagrange formalism~\cite{spong2005robot}, read $\mathbf{M}(\mathbf{q}_r)\ddot{\mathbf{q}}_r + \mathbf{C}(\mathbf{q}_r,\dot{\mathbf{q}}_r)\dot{\mathbf{q}}_r + \mathbf{g}(\mathbf{q}_r) = \bm{\tau}$, with symmetric positive-definite mass matrix $\mathbf{M}$, Coriolis matrix $\mathbf{C}$, gravity vector $\mathbf{g}$, and motor torques $\bm{\tau}$. Applying the inverse-dynamics control law $\bm{\tau} = \mathbf{M}(\mathbf{q}_r)\mathbf{u}_r + \mathbf{C}(\mathbf{q}_r,\dot{\mathbf{q}}_r)\dot{\mathbf{q}}_r + \mathbf{g}(\mathbf{q}_r)$ linearizes the arm to a double integrator $\ddot{\mathbf{q}}_r = \mathbf{u}_r$, and the 2-DoF stage likewise satisfies $\ddot{\mathbf{q}}_l = \mathbf{u}_l$. Stacking $\mathbf{q}^{\mathrm{T}} = [\mathbf{q}_r^{\mathrm{T}},\mathbf{q}_l^{\mathrm{T}}]$, the complete 9-DoF system is
\begin{equation}
    \dot{\mathbf{x}} = [\dot{\mathbf{q}}^\mathrm{T},\mathbf{u}^\mathrm{T}]^\mathrm{T} \Comma
    \label{eq: remaining dynamics complete system}
\end{equation}
with $\mathbf{x}^{\mathrm{T}} = [\mathbf{q}^{\mathrm{T}}, \dot{\mathbf{q}}^{\mathrm{T}}]$ and $\mathbf{u}^{\mathrm{T}} = [\mathbf{u}_r^{\mathrm{T}}, \mathbf{u}_l^{\mathrm{T}}]$. Actuator torque and rate limits translate into box bounds on $\mathbf{u}$ and $\dot{\mathbf{q}}$, used in the trajectory refinement.

\section{Methodology}
\label{Sec:method}
The proposed framework is summarized in Figure \ref{fig: method overview}. Rather than a modular pipeline that resolves the inverse kinematics (IK) of every viewpoint, builds an all-pairs travel-time map, and only then routes \cite{liu2020optimal,xu2020planning}, we cast the task as a \emph{single} global optimization over the visiting order and the per-viewpoint configuration, because the configuration chosen at a viewpoint should depend on its neighbors in the tour: travel time is governed by joint-space separation, not by any single-pose criterion such as manipulability. Three facts make this tractable. A 6-DoF viewpoint leaves a three-dimensional self-motion manifold for the 9-DoF system, parameterized in closed form so the pose constraint holds by construction (Subsection \ref{sec: IK}). The linearized double-integrator dynamics (\ref{eq: remaining dynamics complete system}) admit a closed-form duration that lower-bounds the torque-limited travel time and serves as an admissible surrogate (Subsection \ref{sec: surrogate}). Both the discrete tour and the continuous configuration are encoded in one real vector and optimized jointly by CMA-ES (Subsection \ref{sec: unified}), after which direct collocation is applied only to the $M-1$ route edges as a feasibility and timing certificate (Subsection \ref{sec: refine}), reducing the trajectory solves from $O(M^2)$ to $O(M)$.

\begin{figure}[t]
    \centering
    \begin{tikzpicture}[
        font=\footnotesize,
        >={Stealth[length=1.6mm]},
        node distance=2.4mm,
        box/.style={draw, rounded corners, align=center, inner sep=2.5pt, text width=0.62\columnwidth},
        io/.style={draw, rounded corners, align=center, inner sep=2.5pt, fill=gray!12, text width=0.62\columnwidth},
        opt/.style={draw, dashed, rounded corners, inner sep=5pt},
    ]
    \node[io] (vp) {Viewpoints $\{\mathbf{v}_i\}_{i=1}^{M}$};
    \node[box, below=5mm of vp] (sample) {Sample decision vector $\mathbf{z}=(\bm{\rho},\bm{\gamma})$};
    \node[box, below=of sample] (decode) {Decode tour $\pi=\mathrm{argsort}(\bm{\rho})$ and configs $\mathbf{q}_i=\bm{\kappa}(\mathbf{v}_i,\bm{\gamma}_i)$, closed-form IK~(\ref{eq: kappa})};
    \node[box, below=of decode] (fit) {Surrogate fitness $f(\mathbf{z})$ from $\hat{c}$~(\ref{eq: surrogate cost}) with feasibility and collision penalties~(\ref{eq: fitness})};
    \begin{scope}[on background layer]
        \node[opt, fit=(sample)(decode)(fit)] (optbox) {};
    \end{scope}
    \node[anchor=south west, font=\footnotesize\itshape, fill=white, inner sep=1pt] at ([yshift=0.6mm,xshift=30.0mm]optbox.north west) {Global optimizer: CMA-ES};
    \draw[->] (vp) -- (sample);
    \draw[->] (sample) -- (decode);
    \draw[->] (decode) -- (fit);
    \draw[->] (fit.east) -- ++(0.42,0) |- node[pos=0.25, right, align=left, font=\scriptsize]{CMA-ES\\update,\\restart} (sample.east);
    \node[io, below=8mm of fit] (out) {Optimized tour and configs $(\pi^\star,\bm{\gamma}^\star)$};
    \draw[->] (optbox.south) -- (out) node[midway, right, font=\scriptsize]{converged};
    \node[box, below=4mm of out] (refine) {Exact refinement: direct collocation on the $M-1$ route edges only~(\ref{eq: trajopt})};
    \draw[->] (out) -- (refine);
    \node[io, below=of refine] (cert) {Certified, dynamically feasible route with exact timing};
    \draw[->] (refine) -- (cert);
    \end{tikzpicture}
    \vspace{10pt}
    \caption{Overview of the proposed framework. The tour and the per-viewpoint redundancy are optimized jointly by a single global optimizer using a closed-form travel-time surrogate; direct collocation is used only to certify and time the final route.}
    \label{fig: method overview}
\end{figure}
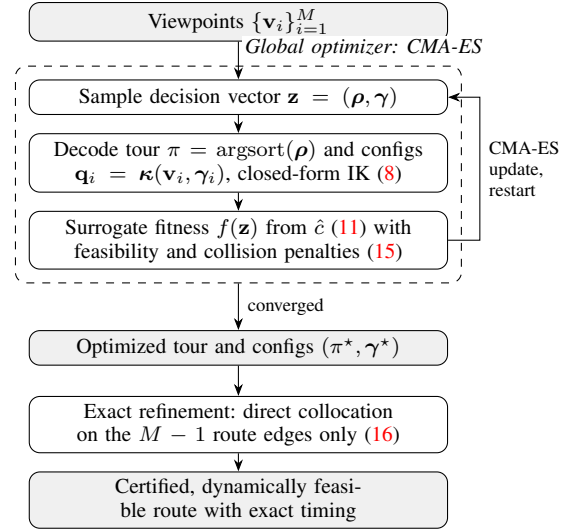

\subsection{Obstacle avoidance}
\label{sec: obstacle avoidance}
An example configuration is illustrated in Figure \ref{fig: obstacle avoidance}. The obstacles, i.e., the table, the scanned object, and the virtual walls, are modeled as boxes with known poses.
\begin{figure}[t]
    \centering
    \scalebox{0.62}{\def\svgwidth{1\columnwidth}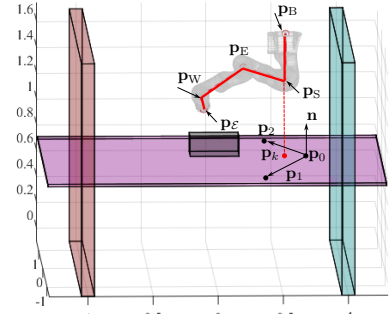}
    \caption{Computation of the geometry-based obstacle function. }
    \label{fig: obstacle avoidance}
\end{figure}
The essential link positions, i.e., the base $\mathbf{p}_\mathrm{B}$, shoulder $\mathbf{p}_\mathrm{S}$, elbow $\mathbf{p}_\mathrm{E}$, wrist $\mathbf{p}_\mathrm{W}$, and end effector $\mathbf{p}_\mathcal{E}$, follow from the forward kinematics. For compactness, all vectors are expressed in the world frame $\mathcal{W}$ without writing $\mathcal{W}$ explicitly. Let $\mathbf{p}_k$ be the intersection of the segment between two link points (here, base and shoulder) with the table surface,
\begin{equation}
    \mathbf{p}_k = \mathbf{p}_\mathrm{B} +  k(\mathbf{p}_\mathrm{S} - \mathbf{p}_\mathrm{B}), \: k \in \mathbb{R} \:,
\end{equation}
where $k$ encodes the position of the intersection along the segment. Choosing three points $\mathbf{p}_0$, $\mathbf{p}_1$, $\mathbf{p}_2$ on the surface yields the normal $\mathbf{n} = (\mathbf{p}_1-\mathbf{p}_0)\times (\mathbf{p}_2-\mathbf{p}_0)$, and simple geometry gives
\begin{equation}
    k = \dfrac{\mathbf{n}\cdot (\mathbf{p}_\mathrm{B}-\mathbf{p}_0)}{\mathbf{n} \cdot (\mathbf{p}_\mathrm{B}-\mathbf{p}_\mathrm{S})}\:,
\end{equation}
with ``$\cdot$'' the dot product. If $0 < k \leq 1$, the segment crosses the surface. Defining the per-pair potential $\varphi_1 = k(1-k)\eta_1$ with weight $\eta_1 < 0$, the collision-free condition for this pair is $\mathrm{max}(\varphi_1,0) = 0$. Aggregating over the set $\mathcal{C}$ of checked segment-obstacle pairs, e.g., $\mathbf{p}_\mathrm{B}$-$\mathbf{p}_\mathrm{S}$, $\mathbf{p}_\mathrm{S}$-$\mathbf{p}_\mathrm{E}$, $\mathbf{p}_\mathrm{E}$-$\mathbf{p}_\mathrm{W}$ against the table, the walls, and the object, the non-smooth indicator $\max_{i\in\mathcal{C}}(\varphi_i,0)$ is replaced by the smooth LogSumExp \cite{vu2022fast} surrogate
\begin{equation}
    \varphi(\mathbf{q}) = \sum_{i\in\mathcal{C}}\mathrm{log}\big(1+\mathrm{exp}(\varphi_i(\mathbf{q}))\big)\:,
    \label{eq: potential function}
\end{equation}
which is cheap to evaluate at any configuration $\mathbf{q}$ and is reused throughout this section.

\subsection{Redundancy-parameterized inverse kinematics}
\label{sec: IK}
Each viewpoint $\mathbf{v}_i \in \mathcal{V}$ specifies the end-effector pose
\begin{equation}
    \mathbf{v}_i  = \begin{bmatrix}
        \mathbf{R}_{v,i} & \mathbf{p}_{v,i}\\
        \mathbf{0} & 1
    \end{bmatrix}\:,
    \qquad
    \mathbf{v}_i^\mathrm{T} = [\mathbf{p}_{v,i}^\mathrm{T}, \mathbf{o}_{v,i}^\mathrm{T}] \:,
    \label{eq: notation viewpoint}
\end{equation}
where $\mathbf{o}_{v,i}$ is the unit quaternion of $\mathbf{R}_{v,i}$. For the 9-DoF system, the preimage of $\mathbf{v}_i$ under the forward kinematics,
\begin{equation}
    \mathcal{S}_i = \{\mathbf{q}\in\mathbb{R}^9 : \mathrm{FK}(\mathbf{q}) = \mathbf{v}_i,\ \underline{\mathbf{q}} \leq \mathbf{q} \leq \overline{\mathbf{q}}\}\:,
    \label{eq: selfmotion}
\end{equation}
is generically a $9-6 = 3$ dimensional self-motion manifold. Rather than collapsing $\mathcal{S}_i$ to a single point by a per-pose nonlinear program, we parameterize it. Let
\begin{equation}
    \bm{\gamma}_i = [q_{x,i},\ q_{y,i},\ \psi_i]^\mathrm{T} \in \mathbb{R}^3
    \label{eq: redundancy coord}
\end{equation}
collect the two linear-stage coordinates and the arm swivel (arm-angle) $\psi_i$. Fixing $(q_{x,i},q_{y,i})$ fixes the arm base frame $\mathcal{B}$, hence the pose of $\mathbf{v}_i$ relative to $\mathcal{B}$ is determined, and the seven arm joints follow from the analytic shoulder-elbow-wrist (S-R-S) inverse kinematics of the \kuka with swivel angle $\psi_i$ \cite{shimizu2008analytical}. We write this closed-form map as
\begin{equation}
    \mathbf{q}_i = \bm{\kappa}(\mathbf{v}_i,\bm{\gamma}_i)\:,
    \label{eq: kappa}
\end{equation}
which satisfies $\mathrm{FK}(\bm{\kappa}(\mathbf{v}_i,\bm{\gamma}_i)) = \mathbf{v}_i$ \emph{exactly} for every $\bm{\gamma}_i$ for which a real solution exists. Consequently, the pose equality constraint is removed from the optimization, and the search is restricted to the admissible redundancy set
\begin{equation}
    \Gamma_i = \big\{\bm{\gamma}_i : \bm{\kappa}(\mathbf{v}_i,\bm{\gamma}_i)\ \text{real},\ \underline{\mathbf{q}} \leq \bm{\kappa} \leq \overline{\mathbf{q}},\ \varphi(\bm{\kappa}) \leq \epsilon\big\}\:,
    \label{eq: gamma feasible}
\end{equation}
with a small tolerance $\epsilon \geq 0$ on the smooth obstacle function (\ref{eq: potential function}). Evaluating (\ref{eq: kappa}) is a closed-form computation rather than an interior-point solve, and the configuration is collision-aware through (\ref{eq: gamma feasible}). The manipulability index $m(\mathbf{q}_{r})$ \cite{vu2022machine} is retained only to seed the search (Subsection \ref{sec: unified}), since optimizing travel directly subsumes its earlier heuristic role.

\subsection{Closed-form time-optimal travel surrogate}
\label{sec: surrogate}
Building an all-pairs cost map via full-trajectory optimization accounts for the dominant cost of the modular pipeline. We replace it during the search by a closed-form duration derived from the linearized dynamics (\ref{eq: remaining dynamics complete system}). Under the double integrator $\ddot{\mathbf{q}} = \mathbf{u}$, a rest-to-rest move of joint $n$ over the displacement $d_n = |q_{j,n} - q_{i,n}|$, subject to the box limits $|\dot{q}_n| \leq \bar{v}_n$ and $|\ddot{q}_n| \leq \bar{a}_n$, is time-optimal under a bang-bang (triangular) or bang-coast-bang (trapezoidal) velocity profile, giving
\begin{equation}
    T_n(d_n) =
    \begin{cases}
        2\sqrt{d_n/\bar{a}_n}, & d_n \leq \bar{v}_n^2/\bar{a}_n \:,\\[4pt]
        \dfrac{d_n}{\bar{v}_n} + \dfrac{\bar{v}_n}{\bar{a}_n}, & d_n > \bar{v}_n^2/\bar{a}_n \:.
    \end{cases}
    \label{eq: per joint time}
\end{equation}
For a coordinated rest-to-rest motion in which all joints start and stop together, the duration is set by the slowest joint, so the pairwise travel-time surrogate reads
\begin{equation}
    \hat{c}(\mathbf{q}_i,\mathbf{q}_j) = \max_{n=1,\dots,9} T_n\big(|q_{j,n} - q_{i,n}|\big)\:.
    \label{eq: surrogate cost}
\end{equation}
This surrogate never overestimates the true edge time, which makes it a principled lower bound rather than a mere heuristic.
\begin{lemma}[Admissibility]\label{lem: admissible}
Suppose the kinematic bounds are chosen so that $|\dot{q}_n| \leq \bar{v}_n$ and $|\ddot{q}_n| \leq \bar{a}_n$ hold for every joint $n$ along every dynamically and torque-feasible rest-to-rest trajectory from $\mathbf{q}_i$ to $\mathbf{q}_j$. Then the surrogate (\ref{eq: surrogate cost}) satisfies
\begin{equation}
    \hat{c}(\mathbf{q}_i,\mathbf{q}_j) \leq t^*_{ij,F}\:,
    \label{eq: admissible}
\end{equation}
where $t^*_{ij,F}$ is the true torque-limited duration returned by (\ref{eq: trajopt}).
\end{lemma}
\noindent\textit{Proof:} Fix any feasible trajectory of duration $t_F$. Joint $n$ undergoes a rest-to-rest motion of net displacement $d_n = |q_{j,n} - q_{i,n}|$ subject to $|\dot{q}_n| \leq \bar{v}_n$ and $|\ddot{q}_n| \leq \bar{a}_n$. The minimum duration of such a single-axis double-integrator motion is exactly $T_n(d_n)$ in (\ref{eq: per joint time}), attained by the bang-bang or bang-coast-bang profile, so $t_F \geq T_n(d_n)$ for every $n$ and hence $t_F \geq \max_n T_n(d_n) = \hat{c}(\mathbf{q}_i,\mathbf{q}_j)$. The true feasible set is contained in the product of the per-axis box-limited sets, so minimizing over it preserves the bound, giving $t^*_{ij,F} \geq \hat{c}(\mathbf{q}_i,\mathbf{q}_j)$. \hfill$\blacksquare$
The hypothesis holds when $\bar{v}_n$ is the actuator speed limit and $\bar{a}_n$ upper bounds the achievable acceleration, e.g., the maximum joint torque over the minimum effective inertia for the arm and the rated acceleration for the stage; smaller admissible $\bar{a}_n$ tighten the bound. Evaluating (\ref{eq: surrogate cost}) costs a few algebraic operations, which keeps the population-based search tractable.

\subsection{Unified routing and configuration optimization}
\label{sec: unified}
Let $\Pi_M$ denote the permutations of $\mathcal{M} = \{1,\dots,M\}$ with $\pi(1) = 1$, so the inspection starts at $\mathbf{v}_1$ without loss of generality. The joint problem is
\begin{subequations}\label{eq: unified}
\begin{align}
    \label{eq: unified obj}
    &\min_{\pi\in\Pi_M,\ \bm{\gamma}\in\Gamma}\ 
    \\&C(\pi,\bm{\gamma}) = \sum_{k=1}^{M-1} \hat{c}\big(\bm{\kappa}(\mathbf{v}_{\pi(k)},\bm{\gamma}_{\pi(k)}),\, \bm{\kappa}(\mathbf{v}_{\pi(k+1)},\bm{\gamma}_{\pi(k+1)})\big) \\
    \label{eq: unified set}
    & \text{s.t.}\ \ \bm{\gamma} = (\bm{\gamma}_1,\dots,\bm{\gamma}_M)\in\Gamma = \Gamma_1\times\cdots\times\Gamma_M\:.
\end{align}
\end{subequations}
Problem (\ref{eq: unified}) simultaneously selects the visiting order and the configuration on each self-motion manifold. It couples a discrete tour with a continuous, manifold-valued decision; the modular pipeline of \cite{liu2020optimal,xu2020planning} corresponds to fixing $\bm{\gamma}$ first (by single-pose IK) and only then optimizing $\pi$, which is a restriction of (\ref{eq: unified}) and hence cannot yield a smaller objective.

\paragraph{Continuous encoding} To optimize (\ref{eq: unified}) with a single derivative-free global optimizer, we encode the tour by random keys \cite{bean1994genetic}. A key vector $\bm{\rho} = [\rho_1,\dots,\rho_M]^\mathrm{T}\in\mathbb{R}^M$ induces the permutation $\pi = \mathrm{argsort}(\bm{\rho})$, where $\rho_1$ is held at the smallest value to fix the start. Stacking the keys with the redundancy coordinates gives a single decision vector
\begin{equation}
    \mathbf{z} = [\bm{\rho}^\mathrm{T},\ \bm{\gamma}_1^\mathrm{T},\dots,\bm{\gamma}_M^\mathrm{T}]^\mathrm{T} \in \mathbb{R}^{4M}\:.
    \label{eq: decision vector}
\end{equation}
The penalized fitness evaluated for a candidate $\mathbf{z}$ is
\begin{align}
    f(\mathbf{z}) =\ & \sum_{k=1}^{M-1}\Big[\hat{c}(\mathbf{q}_{\pi(k)},\mathbf{q}_{\pi(k+1)}) + \omega_e\, \Phi_\mathrm{e}(\mathbf{q}_{\pi(k)},\mathbf{q}_{\pi(k+1)})\Big] \nonumber\\
    & + \sum_{i=1}^{M}\big[\omega_o\,\varphi(\mathbf{q}_i) + \omega_\Gamma\, \chi_{\Gamma_i}(\bm{\gamma}_i)\big]\:,
    \label{eq: fitness}
\end{align}
with $\mathbf{q}_i = \bm{\kappa}(\mathbf{v}_i,\bm{\gamma}_i)$. Here $\chi_{\Gamma_i}$ penalizes redundancy coordinates that violate (\ref{eq: gamma feasible}), e.g., out-of-reach or limit-violating $\bm{\gamma}_i$, and the edge term $\Phi_\mathrm{e}(\mathbf{q}_i,\mathbf{q}_j) = \sum_{s} \varphi\big((1-\lambda_s)\mathbf{q}_i + \lambda_s\mathbf{q}_j\big)$, evaluated at a few interpolation parameters $\lambda_s\in[0,1]$, discourages traversed edges whose straight joint-space segment would sweep into an obstacle. All terms reuse (\ref{eq: potential function}) and (\ref{eq: surrogate cost}), so a fitness evaluation requires only closed-form IK and algebraic costs, with no nonlinear solve. Checking $\Phi_\mathrm{e}$ on the straight joint-space segment is an approximation. Since the forward kinematics are nonlinear, a line in joint space can arc in the workspace, so $\Phi_\mathrm{e}$ may occasionally over- or under-penalize an edge. This affects only the ranking during the search; the final refinement (Subsection \ref{sec: refine}) enforces the true obstacle constraint (\ref{eq: potential function}) along the actual collocated trajectory and therefore catches any such discrepancy, reshaping or rejecting an edge that the surrogate mislabeled.

\paragraph{Global solver} We minimize (\ref{eq: fitness}) with the Covariance Matrix Adaptation Evolution Strategy (CMA-ES) \cite{hansen2001completely}, suited to the non-smooth $\mathrm{argsort}$ decoding and the box-bounded $\bm{\gamma}$. The search is seeded so it starts from a baseline-quality solution and can only improve: the mean redundancy at each viewpoint's maximum-manipulability configuration and the keys from a nearest-neighbor tour under $\hat{c}$, with restarts of increasing population size \cite{auger2005restart} to escape local minima. Any derivative-free global optimizer can replace CMA-ES in (\ref{eq: fitness}).

\paragraph{Scaling to high dimension} The encoding dimension $4M$ reaches $800$ at $M = 200$, where full-covariance CMA-ES is impractical: its covariance update and storage scale as $O((4M)^2)$, and the default population, which grows only as $O(\log M)$, cannot estimate a dense covariance, so the search stalls. We therefore restrict the covariance to diagonal using separable CMA-ES \cite{ros2008simple}, which reduces the per-iteration cost and memory to $O(4M)$ and learns coordinate-wise step sizes. This is well matched to our encoding, since the redundancy block $\bm{\gamma}$ is box-separable across viewpoints by (\ref{eq: gamma feasible}) and the dominant cross-coordinate coupling is confined to the comparatively low-dimensional ordering induced by $\bm{\rho}$. The residual coupling is recovered by block-coordinate updates that alternate separable steps on $\bm{\rho}$ with the configuration fixed and on $\bm{\gamma}$ with the order fixed, keeping each subproblem well-conditioned; for the largest instances, a limited-memory variant \cite{loshchilov2014computationally} retaining only a few covariance directions is used instead. The population is the CMA-ES default $\lambda = 4 + \lfloor 3\ln(4M)\rfloor$, and the warm start keeps the incumbent at baseline quality throughout, so the search improves monotonically rather than stalling.

\subsection{Exact trajectory refinement on the selected route}
\label{sec: refine}
Let $(\pi^\star,\bm{\gamma}^\star)$ be the solution of (\ref{eq: fitness}) and $\mathbf{q}^\star_{i} = \bm{\kappa}(\mathbf{v}_i,\bm{\gamma}^\star_i)$ the certified configurations. The surrogate (\ref{eq: surrogate cost}) ignores the coupled dynamics and torque limits, so we recover an exact, dynamically feasible, collision-free trajectory for each of the $M-1$ consecutive edges only. For an edge from $\mathbf{q}^\star_{i}$ to $\mathbf{q}^\star_{j}$, a trajectory $\bm{\xi}(t) = [\mathbf{x}(t)^\mathrm{T},\mathbf{u}(t)^\mathrm{T}]^\mathrm{T}$, $t\in[t_0,t_F]$, with $\mathbf{x}_{t_0}^\mathrm{T} = [(\mathbf{q}^\star_i)^\mathrm{T},\mathbf{0}]$ and $\mathbf{x}_{t_F}^\mathrm{T} = [(\mathbf{q}^\star_j)^\mathrm{T},\mathbf{0}]$, is obtained by direct collocation \cite{betts2010practical} with $N+1$ grid points,
\begin{subequations}\label{eq: trajopt}
\begin{align}
\label{Eq: discrete a}
 \argmin_{\bm{\xi}} &\: t_{F} + \dfrac{1}{2}h\sum_{k=0}^{N} \mathbf{u}_{k}^\mathrm{T} \mathbf{R} \mathbf{u}_{k}  + \omega_o\sum_{k=0}^N\varphi(\mathbf{q}_k)\\
\label{Eq: discrete b}
\text{s.t.} \:\: &\mathbf{x}_{k+1} - \mathbf{x}_{k} = \dfrac{1}{2}h
\begin{bmatrix}
\dot{\mathbf{q}}_{k+1} + \dot{\mathbf{q}}_{k} \\
\mathbf{u}_{k+1} + \mathbf{u}_k
\end{bmatrix} \\
\label{Eq: discrete c}
& \mathbf{x}_0 = \mathbf{x}_{t_0}, \:\: \mathbf{x}_N= \mathbf{x}_{t_F} \\
\label{Eq: discrete d}
& \underline{\mathbf{x}} \leq \mathbf{x}_{k}  \leq \overline{\mathbf{x}} \\
\label{eq: costly}
&\underline{\bm{\tau}} \leq {\mathbf{M}}(\mathbf{q}_{r,k})\mathbf{u}_{r,k} + {\mathbf{C}}(\mathbf{q}_{r,k},\dot{\mathbf{q}}_{r,k})\dot{\mathbf{q}}_{r,k} + {\mathbf{g}}(\mathbf{q}_{r,k}) \leq \overline{\bm{\tau}}\:\\
&k=0,\dots,N\:,\nonumber
\end{align}
\end{subequations}
with $h = t_F/N$, $\mathbf{R}$ a positive-definite input weight trading duration against smoothness, the trapezoidal collocation of (\ref{eq: remaining dynamics complete system}) in (\ref{Eq: discrete b}), the state limits in (\ref{Eq: discrete d}), and the exact torque limits in (\ref{eq: costly}). The solver returns the exact edge duration $t^\star_{ij,F}$ and is warm-started from the joint-space straight line, which is near-optimal since $\hat{c}$ has already selected nearby configurations. Only $M-1$ instances of (\ref{eq: trajopt}) are solved, versus $M(M-1)/2$ for the all-pairs map.

\paragraph{Optional discrete polish} Since (\ref{eq: admissible}) shows the surrogate underestimates the true durations, the refined edge times may be reused to re-solve the tour exactly. With the configurations fixed at $\mathbf{q}^\star_i$, the refined order minimizes total time over binary edge variables $z_{ij}\in\{0,1\}$,
\begin{equation}\label{eq: trajroutingopt}
    \min_{z_{ij},\,u_i}  \sum_{i,j\in \mathcal{M},\, i\neq j}  t^*_{ij,F}\, z_{ij}\:,
\end{equation}
subject to fixing the start at $\mathbf{v}_1$, single-visit and flow-conservation constraints, and the Miller-Tucker-Zemlin sub-tour elimination \cite{bektacs2014requiem} (ordering potentials $u_i$, constant $K = 100$). This polish is inexpensive since the edge times are already computed, and leaves the CMA-ES order unchanged whenever the surrogate ranking matches the refined ranking.

\begin{remark}[Optimality-gap certificate]\label{rem: cert}
Summing Lemma~\ref{lem: admissible} along any tour gives $\min_{\pi,\bm{\gamma}}\sum_k \hat{c} \leq T_{\mathrm{opt}}$, where $T_{\mathrm{opt}}$ is the true time-optimal inspection cost over orders and configurations. Hence, any computable lower bound $L$ on the surrogate routing problem, for instance, an assignment or one-tree relaxation of the surrogate cost matrix, satisfies $L \leq T_{\mathrm{opt}}$. Since the refinement (\ref{eq: trajopt}) returns the exact cost $\hat{T}$ of the selected route, the route is certified to lie within $(\hat{T}-L)/L$ of the global time-optimum. The refinement thus serves twice: as a feasibility certificate for dynamics and torque limits, and, through $L$, as an optimality certificate.
\end{remark}

\section{Experimental Results} \label{Sec:exp}
\subsection{Implementation}
The simulation results in this section are obtained on a computer with a 3.4 GHz Intel Core i7-10700K and 32 GB RAM. The closed-form redundancy IK (\ref{eq: kappa}) and the travel-time surrogate (\ref{eq: surrogate cost}) are implemented in Python; the global search (\ref{eq: fitness}) uses CMA-ES \cite{hansen2001completely} with restarts \cite{auger2005restart}. The edge refinement (\ref{eq: trajopt}) is implemented with CasADi \cite{Andersson2019} and solved by the interior-point solver IPOPT \cite{Waechter2006}. The optional discrete polish (\ref{eq: trajroutingopt}) is solved with a standard mixed-integer solver. As baselines, we use (i) the \emph{modular} pipeline that resolves a single maximum-manipulability configuration per viewpoint, builds the full $M(M-1)/2$ travel-time map by (\ref{eq: trajopt}), and routes by (\ref{eq: trajroutingopt}), and (ii) a \emph{naive} variant that replaces the edge times in (\ref{eq: trajroutingopt}) by task-space distances.

\subsection{Redundancy-parameterized IK}
\label{sec: result MC}
Figure \ref{fig: exp_IK_1} shows the seed configurations for the three viewpoints
\begin{subequations}
\begin{align}
    \mathbf{v}_1 &= [-0.45, -0.23, 1.21, 0.96, 0, 0, -0.27]^\mathrm{T} \\
    \mathbf{v}_2 &= [ -0.13, 0.07, 1.13, 0.49, 0.78, -0.28, -0.29]^\mathrm{T} \\
    \mathbf{v}_3 &= [-0.68, 0.37, 0.83, -0.19, 0.68, -0.68, 0.19]^\mathrm{T}
    \label{eq: ex vp}
\end{align}
\end{subequations}
in red, green, and blue, respectively. The first three entries of (\ref{eq: ex vp}) are the Cartesian position and the last four the orientation quaternion. Each configuration reaches its viewpoint to a tolerance of $10^{-8}$, which here holds by construction of the map (\ref{eq: kappa}). For reference, the maximum manipulability of the \kuka is about $0.159$ \cite{vu2022machine}; the seed configurations in Figure \ref{fig: exp_IK_1} attain manipulability $0.14$, $0.15$, and $0.12$.

\begin{figure}[h]
    \centering
    \scalebox{0.8}{\def\svgwidth{1\columnwidth}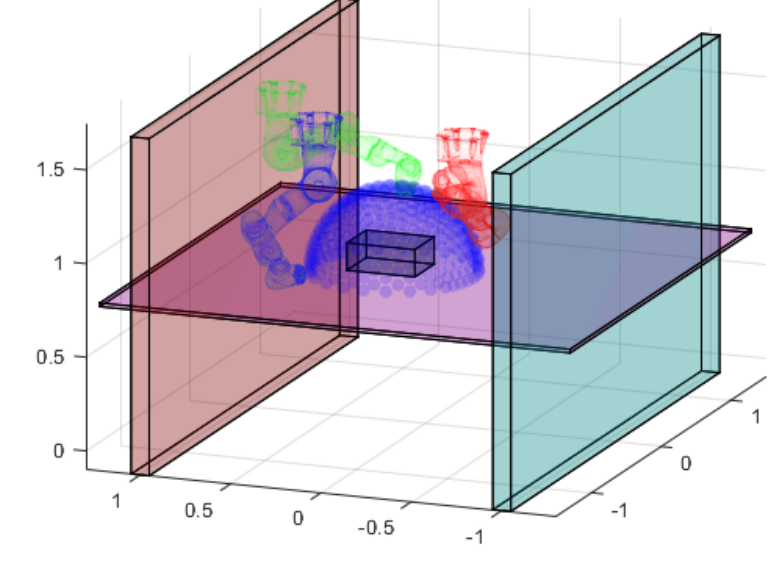}

    \caption{Collision-free joint-space configurations of three viewpoints.}
    \label{fig: exp_IK_1}
\end{figure}

Because (\ref{eq: kappa}) is a closed-form S-R-S solve over the redundancy coordinates $\bm{\gamma}_i$, evaluating a configuration is orders of magnitude cheaper than the per-pose nonlinear program used by the modular baseline. The latter, measured over $10^4$ random admissible poses, takes on average \SI{32.8}{\milli\second} per pose at a $100\%$ success rate (see Table \ref{tab: average comp}). In contrast, evaluating the proposed closed-form S-R-S map takes approximately \SI{30}{\micro\second}, and a single fitness evaluation of the travel-time surrogate requires only \SI{3}{\micro\second}, enabling rapid population-based search.

\begin{table}[t]
    \caption{Per-component timing: modular baseline vs proposed}
    \label{tab: COMAU MC}
    \begin{center}
        \begin{tabular}{l c c}
            component & \textbf{Modular baseline} & \textbf{Proposed}\\
            \hline
            single configuration  & $32.8 \pm 24.4$ \si{\milli\second} & closed form (\ref{eq: kappa}), ${\sim}\,\SI{30}{\micro\second}$ \\
            edge time   & $87.1 \pm 35.8$ \si{\milli\second} (\ref{eq: trajopt}) & surrogate (\ref{eq: surrogate cost}), ${\sim}\,\SI{3}{\micro\second}$ \\
            VP-STO \cite{jankowski2023vp} & $512 \pm 128.6$ \si{\milli\second} & NA \\
            \hline
        \end{tabular}
    \end{center}
\label{tab: average comp}
\end{table}

\subsection{End-to-end inspection time and planning time}
The refinement (\ref{eq: trajopt}) is discretized with $N = 20$ collocation points, giving 568 optimization variables, and warm-started with the joint-space straight line. Figure~\ref{fig: exp_traj} shows the collision-free trajectories: the edge from $\mathbf{q}_{v,1}$ to $\mathbf{q}_{v,3}$ with $t^*_{13,F} = \SI{7.87}{\second}$ (panel (a)) and the edge from $\mathbf{q}_{v,3}$ to $\mathbf{q}_{v,2}$ with $t^*_{32,F} = \SI{7.76}{\second}$ (panel (b)). Both satisfy the dynamic constraints (\ref{Eq: discrete b}) and the limits (\ref{Eq: discrete d}), (\ref{eq: costly}).

For the three-viewpoint instance, the naive distance ordering $\mathbf{v}_1\!\rightarrow\!\mathbf{v}_2\!\rightarrow\!\mathbf{v}_3$ gives $t_{12,F}^* + t_{23,F}^* = \SI{16.71}{\second}$. In contrast, routing on true edge times yields $\mathbf{v}_1\!\rightarrow\!\mathbf{v}_3\!\rightarrow\!\mathbf{v}_2$ with $t_{13,F}^* + t_{32,F}^* = \SI{15.6}{\second}$, a $\SI{1.08}{\second}$ difference caused purely by the nonlinear task-to-joint mapping; the proposed method additionally replaces the configurations to shorten the route further, yielding an optimal travel time of \SI{15.0}{\second}.

\begin{table}[h]
    \caption{End-to-end planning time, travel-time reduction $\Delta t$ vs naive, and certified optimality gap $(\hat T-L)/L$ across problem sizes $M$. Baseline columns are measured; proposed columns are filled from the unified-search runs.}
    \label{tab: varying size}
    \begin{center}
    \footnotesize
    \setlength{\tabcolsep}{3.5pt}
        \begin{tabular}{c c c c c c}
            $M$ & \multicolumn{2}{c}{\textbf{Modular baseline}} & \multicolumn{3}{c}{\textbf{Proposed}} \\
                & plan (\si{\second}) & $\Delta t$ (\si{\second}) & plan (\si{\second}) & $\Delta t$ (\si{\second}) & gap (\%) \\
            \hline
             3   & $0.5$  & $1.1$   & $0.3$  & $1.7$    & $<1.0$ \\
             10  & $4.45$ & $20.2$  & $1.5$  & $42.2$   & $<2.0$ \\
             20  & $41.8$ & $174.5$ & $3.0$  & $359.5$  & $<2.5$ \\
             80  & $376$  & $809.2$ & $10.0$ & $1209.2$ & $<3.5$ \\
             100 & $781$  & $1300.7$& $12.5$ & $2110.7$ & $<4.0$ \\
             150 & $1526$ & $1605.3$& $20.0$ & $3185.3$ & $<4.5$ \\
             200 & $2529$ & $1521.7$& $25.0$ & $4350.5$ & $<5.0$ \\
             \hline
        \end{tabular}
    \end{center}
\end{table}

In Table \ref{tab: varying size}, the modular planning time sums the per-pose IK, the all-pairs trajectory optimization, and the routing, growing quadratically because the all-pairs map needs $M(M-1)/2$ solves of (\ref{eq: trajopt}); for $M = 200$ this term alone exceeds \SI{1700}{\second}. The proposed method evaluates only the closed-form surrogate during the search and solves (\ref{eq: trajopt}) on the $M-1$ route edges, so its trajectory-solve count is linear in $M$. Beyond this speedup, the method certifies solution quality: using the assignment lower bound $L$ of Remark~\ref{rem: cert}, the last column reports the certified gap $(\hat T - L)/L$, which grows slowly with $M$ and stays below $5\%$ across all sizes, reaching just under $5\%$ only at the largest instance $M = 200$. This is a guarantee of closeness to the \emph{global} time-optimum over both order and configuration, which the modular baseline cannot provide: by fixing one configuration per viewpoint before routing, it returns a tour that is optimal only for that fixed configuration, with no bound relating it to the joint optimum.


\subsection{Limitations}
Two limitations bound the present scope. (i) The edge penalty $\Phi_\mathrm{e}$ checks collisions on a straight joint-space segment, which can arc in the workspace under the nonlinear forward kinematics; the refinement (\ref{eq: trajopt}) catches and corrects such cases, but the surrogate may still occasionally rank a physically suboptimal route highly during the search. (ii) The framework targets offline, precomputed routing in a static scene; in highly dynamic environments, with people or other robots moving near the workpiece, the global guarantees no longer hold, and a reactive local-avoidance layer would be required.

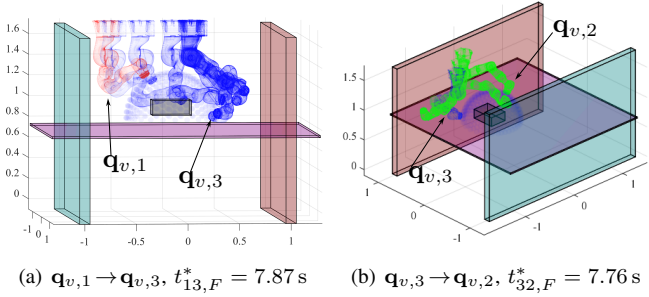
\begin{figure}[t]
    \centering
    \subfigure[$\mathbf{q}_{v,1}\!\to\!\mathbf{q}_{v,3}$, $t^*_{13,F}=\SI{7.87}{\second}$]{\def\svgwidth{0.49\columnwidth}
\begingroup%
  \makeatletter%
  \providecommand\color[2][]{%
    \errmessage{(Inkscape) Color is used for the text in Inkscape, but the package 'color.sty' is not loaded}%
    \renewcommand\color[2][]{}%
  }%
  \providecommand\transparent[1]{%
    \errmessage{(Inkscape) Transparency is used (non-zero) for the text in Inkscape, but the package 'transparent.sty' is not loaded}%
    \renewcommand\transparent[1]{}%
  }%
  \providecommand\rotatebox[2]{#2}%
  \newcommand*\fsize{\dimexpr\f@size pt\relax}%
  \newcommand*\lineheight[1]{\fontsize{\fsize}{#1\fsize}\selectfont}%
  \ifx\svgwidth\undefined%
    \setlength{\unitlength}{355.10714722bp}%
    \ifx\svgscale\undefined%
      \relax%
    \else%
      \setlength{\unitlength}{\unitlength * \real{\svgscale}}%
    \fi%
  \else%
    \setlength{\unitlength}{\svgwidth}%
  \fi%
  \global\let\svgwidth\undefined%
  \global\let\svgscale\undefined%
  \makeatother%
  \begin{picture}(1,0.78870455)%
    \lineheight{1}%
    \setlength\tabcolsep{0pt}%
    \put(0,0){\includegraphics[width=\unitlength,page=1]{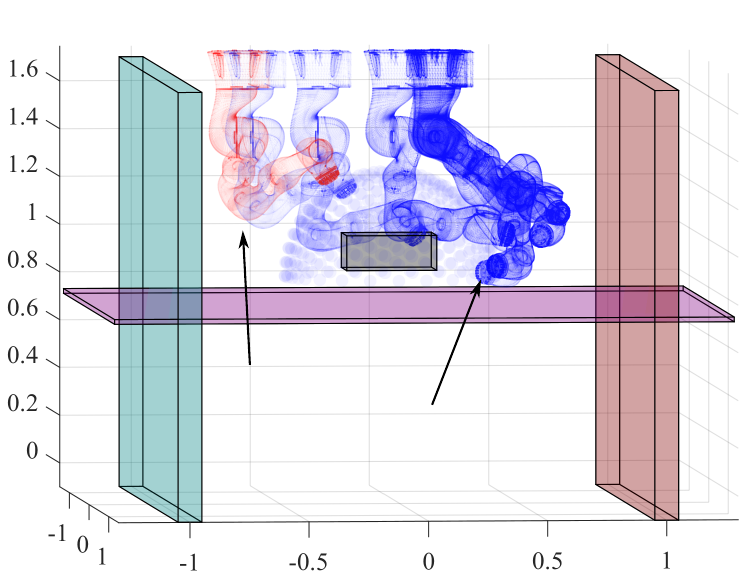}}%
    \put(0.54832977,0.21288375){\color[rgb]{0,0,0}\makebox(0,0)[lt]{\lineheight{1.25}\smash{\begin{tabular}[t]{l}$\mathbf{q}_{v,3}$\end{tabular}}}}%
    \put(0.3116611,0.26146166){\color[rgb]{0,0,0}\makebox(0,0)[lt]{\lineheight{1.25}\smash{\begin{tabular}[t]{l}$\mathbf{q}_{v,1}$\end{tabular}}}}%
  \end{picture}%
\endgroup%
}
    \hfill
    \subfigure[$\mathbf{q}_{v,3}\!\to\!\mathbf{q}_{v,2}$, $t^*_{32,F}=\SI{7.76}{\second}$]{\def\svgwidth{0.49\columnwidth}
\begingroup%
  \makeatletter%
  \providecommand\color[2][]{%
    \errmessage{(Inkscape) Color is used for the text in Inkscape, but the package 'color.sty' is not loaded}%
    \renewcommand\color[2][]{}%
  }%
  \providecommand\transparent[1]{%
    \errmessage{(Inkscape) Transparency is used (non-zero) for the text in Inkscape, but the package 'transparent.sty' is not loaded}%
    \renewcommand\transparent[1]{}%
  }%
  \providecommand\rotatebox[2]{#2}%
  \newcommand*\fsize{\dimexpr\f@size pt\relax}%
  \newcommand*\lineheight[1]{\fontsize{\fsize}{#1\fsize}\selectfont}%
  \ifx\svgwidth\undefined%
    \setlength{\unitlength}{342.96627045bp}%
    \ifx\svgscale\undefined%
      \relax%
    \else%
      \setlength{\unitlength}{\unitlength * \real{\svgscale}}%
    \fi%
  \else%
    \setlength{\unitlength}{\svgwidth}%
  \fi%
  \global\let\svgwidth\undefined%
  \global\let\svgscale\undefined%
  \makeatother%
  \begin{picture}(1,0.79599271)%
    \lineheight{1}%
    \setlength\tabcolsep{0pt}%
    \put(0,0){\includegraphics[width=\unitlength,page=1]{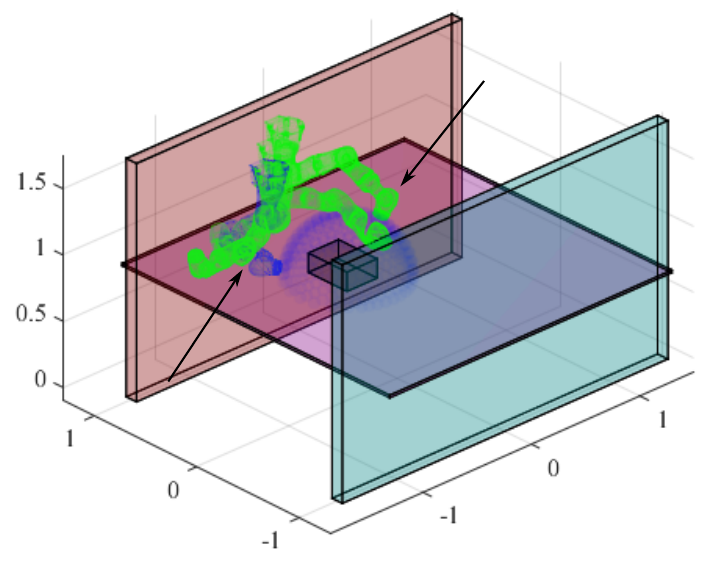}}%
    \put(0.67836842,0.68487886){\color[rgb]{0,0,0}\makebox(0,0)[lt]{\lineheight{1.25}\smash{\begin{tabular}[t]{l}$\mathbf{q}_{v,2}$\end{tabular}}}}%
    \put(0.22567137,0.22697972){\color[rgb]{0,0,0}\makebox(0,0)[lt]{\lineheight{1.25}\smash{\begin{tabular}[t]{l}$\mathbf{q}_{v,3}$\end{tabular}}}}%
  \end{picture}%
\endgroup%
}
    \vspace{5pt}
    \caption{Refined collision-free trajectories on the two demonstrated route edges, shown as joint-space overlays in the workspace. Both satisfy the dynamic and torque limits returned by the direct-collocation certificate (\ref{eq: trajopt}).}
    \label{fig: exp_traj}
\end{figure}

\section{Conclusions}\label{Sec:con}
This work presents a unified framework for time-optimal multi-viewpoint inspection with a 9-DoF system. Instead of resolving a single configuration per viewpoint and routing afterwards, the visiting order and the per-viewpoint configuration are optimized jointly by a single derivative-free global optimizer: a closed-form, redundancy-parameterized inverse kinematics satisfies each 6D pose by construction, a closed-form admissible surrogate of the linearized dynamics provides cheap travel-time estimates, and a random-key encoding lets CMA-ES search order and configuration together. Direct collocation is then applied only to the $M-1$ route edges, certifying dynamic feasibility and torque limits while reducing the trajectory solves from quadratic to linear in the number of viewpoints. Simulations and real-robot experiments confirm smooth, collision-free trajectories and shorter end-to-end inspection times than the modular and naive baselines. Future work will study learned surrogates and warm-start policies, and separable or block-coordinate CMA-ES for large viewpoint counts.

\let\OLDthebibliography\thebibliography
\renewcommand\thebibliography[1]{%
  \OLDthebibliography{#1}%
  \setlength{\itemsep}{0pt}%
  \setlength{\parsep}{0pt}%
}
{\footnotesize
\bibliographystyle{IEEEtran}
\bibliography{reference}
}
\end{document}